\documentclass[letterpaper, 10 pt, conference]{ieeeconf}

\IEEEoverridecommandlockouts                              
                                                          
\pdfoutput=1
\usepackage{comment}
\usepackage[utf8]{inputenc}
 
\usepackage{amsthm}
\usepackage{amsmath,amssymb,amsfonts,}
\usepackage{xspace}
\usepackage[dvipsnames]{xcolor}
\usepackage{algorithm,algcompatible}

\algnewcommand\algorithmicto{\textbf{to}}
\algnewcommand\RETURN{\State \textbf{return} }
\usepackage[noend]{algpseudocode}
\usepackage{mathtools}

\usepackage{sidecap}

\usepackage{tikz,pgfplots}
\usetikzlibrary{intersections}  
\usepgfplotslibrary{fillbetween}

\makeatletter
\let\NAT@parse\undefined
\makeatother
\usepackage{hyperref}

\title{\LARGE \bf{Refined Analysis of Asymptotically-Optimal \\ Kinodynamic Planning in the State-Cost Space}}

\author{Michal Kleinbort$^{1}$, Edgar Granados$^{2}$, Kiril Solovey$^{3}$, Riccardo Bonalli$^{3}$, Kostas E. Bekris$^{2}$, and Dan Halperin$^{1}$
	\thanks{Work by D.H. and M.K. has been supported in part by the Israel Science Foundation
(grant nos.~825/15,1736/19), by the Blavatnik Computer Science Research Fund,
and by grants from Yandex and from Facebook. 
K.B. was supported by NSF awards IIS-1734492, IIS-1723869, CCF-1934924.}
	\thanks{$^{1}$Blavatnik School of Computer Science, Tel-Aviv University, Israel.}%
	\thanks{$^{2}$Computer Science Department, Rutgers University, NJ~08854, USA.}%
	\thanks{$^{3}$Aeronautics and Astronautics Department, Stanford University, CA~94305, USA.}%
}

\pgfplotsset{compat=1.14}
\begin{document}

\maketitle
\thispagestyle{empty}
\pagestyle{empty}

\newcommand{\cupdot}{\mathbin{\mathaccent\cdot\cup}}

\newcommand{\ignore}[1]{}

\def\vor{\textup{Vor}}

\def\P{\mathcal{P}} \def\C{\mathcal{C}} \def\H{\mathcal{H}}
\def\F{\mathcal{F}} \def\U{\mathcal{U}} \def\L{\mathcal{L}}
\def\O{\mathcal{O}} \def\I{\mathcal{I}} \def\S{\mathcal{S}}
\def\G{\mathcal{G}} \def\Q{\mathcal{Q}} \def\I{\mathcal{I}}
\def\T{\mathcal{T}} \def\L{\mathcal{L}} \def\N{\mathcal{N}}
\def\V{\mathcal{V}} \def\B{\mathcal{B}} \def\D{\mathcal{D}}
\def\W{\mathcal{W}} \def\R{\mathcal{R}} \def\M{\mathcal{M}}
\def\X{\mathcal{X}} \def\A{\mathcal{A}} \def\Y{\mathcal{Y}}
\def\L{\mathcal{L}}

\def\dS{\mathbb{S}} \def\dT{\mathbb{T}} \def\dC{\mathbb{C}}
\def\dG{\mathbb{G}} \def\dD{\mathbb{D}} \def\dV{\mathbb{V}}
\def\dH{\mathbb{H}} \def\dN{\mathbb{N}} \def\dE{\mathbb{E}}
\def\dR{\mathbb{R}} \def\dM{\mathbb{M}} \def\dm{\mathbb{m}}
\def\dB{\mathbb{B}} \def\dI{\mathbb{I}} \def\dM{\mathbb{M}}
\def\dZ{\mathbb{Z}}

\def\E{\mathbf{E}} 

\def\eps{\varepsilon}

\def\limn{\lim_{n\rightarrow \infty}}

\def\Reals{\mathbb{R}}
\def\Naturals{\mathbb{N}}
\renewcommand{\leq}{\leqslant}
\renewcommand{\geq}{\geqslant}
\newcommand{\compl}{\mathrm{Compl}}
\def\eqq{\coloneqq}

\newcommand{\stl}{\textsc{Stl}\xspace}
\newcommand{\boost}{\textsc{Boost}\xspace}
\newcommand{\core}{\textsc{Core}\xspace}
\newcommand{\leda}{\textsc{Leda}\xspace}
\newcommand{\cgal}{\textsc{Cgal}\xspace}
\newcommand{\qt}{\textsc{Qt}\xspace}
\newcommand{\gmp}{\textsc{Gmp}\xspace}

\newcommand{\ch}{\mathrm{ch}}
\newcommand{\pspace}{{\sc pspace}\xspace}
\newcommand{\threesum}{{\sc 3Sum}\xspace}
\newcommand{\np}{{\sc np}\xspace}
\newcommand{\degree}{\ensuremath{^\circ}}
\newcommand{\argmin}{\operatornamewithlimits{argmin}}

\newcommand{\dist}{\textup{dist}}

\newcommand{\Cfree}{\C_{\textup{free}}}
\newcommand{\Cforb}{\C_{\textup{forb}}}

\newtheorem{lemma}{Lemma}
\newtheorem{theorem}{Theorem}
\newtheorem{corollary}{Corollary}
\newtheorem{claim}{Claim}
\newtheorem{proposition}{Proposition}
\newtheorem{assumption}{Assumption}

\theoremstyle{definition}
\newtheorem{definition}{Definition}
\newtheorem{remark}{Remark}
\newtheorem{observation}{Observation}

\def\Im{\textup{Im}}

\def\aas{a.a.s.\xspace}
\def\0{\bm{0}}

\makeatletter
\def\thmhead@plain#1#2#3{%
  \thmname{#1}\thmnumber{\@ifnotempty{#1}{ }\@upn{#2}}%
  \thmnote{ {\the\thm@notefont#3}}}
\let\thmhead\thmhead@plain
\makeatother

\def\todo#1{\textcolor{blue}{\textbf{TODO:} #1}}
\def\new#1{\textcolor{magenta}{#1}}
\def\kiril#1{\textcolor{ForestGreen}{\textbf{Kiril:} #1}}
\def\old#1{\textcolor{red}{#1}}
\def\michal#1{\textcolor{red}{\textbf{Michal:} #1}}
\def\riccardo#1{\textcolor{Blue}{\textbf{Riccardo:} #1}}

\def\removed#1{\textcolor{green}{#1}}

\def\dt{\,\mathrm{d}t}
\def\dx{\,\mathrm{d}x}
\def\dy{\,\mathrm{d}y}
\def\drho{\,\mathrm{d}\rho}

\newcommand{\prm}{{\tt PRM}\xspace}
\newcommand{\prmstar}{{\tt PRM}$^*$\xspace}
\newcommand{\rrt}{{\tt RRT}\xspace}
\newcommand{\est}{{\tt EST}\xspace}
\newcommand{\grrt}{{\tt GEOM-RRT}\xspace}
\newcommand{\rrtstar}{{\tt RRT}$^*$\xspace}
\newcommand{\rrg}{{\tt RRG}\xspace}
\newcommand{\btt}{{\tt BTT}\xspace}
\newcommand{\fmt}{{\tt FMT}$^*$\xspace}
\newcommand{\dfmt}{{\tt DFMT}$^*$\xspace}
\newcommand{\dprm}{{\tt DPRM}$^*$\xspace}
\newcommand{\mstar}{{\tt M}$^*$\xspace}
\newcommand{\drrtstar}{{\tt dRRT}$^*$\xspace}
\newcommand{\sst}{{\tt SST}\xspace}
\newcommand{\sststar}{{\tt SST}$^*$\xspace}
\newcommand{\stride}{{\tt STRIDE}\xspace}
\newcommand{\aorrt}{{\tt AO-RRT}\xspace}
\newcommand{\aorrtrebuilding}{{\tt Multi-tree AO-RRT}\xspace}
\newcommand{\aorrtnopruning}{{\tt AO-RRT}\xspace}
\newcommand{\aoest}{{\tt AO-EST}\xspace}
\newcommand{\kpiece}{{\tt KPIECE}\xspace}
\newcommand{\hybrrttwo}{{\tt HybAO-RRT}\xspace}
\newcommand{\hybrrttwostride}{{\tt HybAO-RRT-STRIDE}\xspace}
\newcommand{\hybrrttwoest}{{\tt HybRRT2\!.\!0-EST}\xspace}
\newcommand{\rrttwo}{{\tt AO-RRT}\xspace}
\newcommand{\aorrtprune}{{\tt AO-RRT Pruning}\xspace}
\newcommand{\rrtbc}{{\tt BCRRT}\xspace}
\newcommand{\rrtbctwo}{{\tt BCRRT2\!.\!0}\xspace}

\newcommand{\xmin}{x_{\textup{min}}}
\newcommand{\Xnear}{X_{\textup{near}}}
\newcommand{\Xgoal}{\X_{\textup{goal}}}
\newcommand{\xgoal}{x_{\textup{goal}}}
\newcommand{\xinit}{x_{\textup{init}}}
\newcommand{\xnew}{x_{\textup{new}}}
\newcommand{\xnear}{x_{\textup{near}}}
\newcommand{\xrand}{x_{\textup{rand}}}
\newcommand{\xrandtwo}{x_{\textup{rand2}}}
\newcommand{\ygoal}{y_{\textup{goal}}}
\newcommand{\yinit}{y_{\textup{init}}}
\newcommand{\ynew}{y_{\textup{new}}}
\newcommand{\ynear}{y_{\textup{near}}}
\newcommand{\yrand}{y_{\textup{rand}}}
\newcommand{\ymin}{y_{\textup{min}}}
\newcommand{\xparent}{x_{\textup{parent}}}
\newcommand{\cmin}{c_{\textup{min}}}
\newcommand{\cmax}{c_{\textup{max}}}
\newcommand{\crand}{c_{\textup{rand}}}
\newcommand{\cnew}{c_{\textup{new}}}
\newcommand{\cnear}{c_{\textup{near}}}
\newcommand{\Tprop}{T_{\textup{prop}}}
\newcommand{\trand}{t_{\textup{rand}}}
\newcommand{\tnew}{t_{\textup{new}}}
\newcommand{\urand}{u_{\textup{rand}}}
\newcommand{\unew}{u_{\textup{new}}}
\newcommand{\pinew}{\pi_{\textup{new}}}
\newcommand{\pimin}{\pi_{\textup{min}}}

\newcommand{\randomstate}{\textsc{random-state}}
\newcommand{\sample}{\textsc{sample}}
\newcommand{\nearest}{\textsc{nearest}}
\newcommand{\near}{\textsc{near}}
\newcommand{\steer}{\textsc{steer}}
\newcommand{\collisionfree}{\textsc{collision-free}}
\newcommand{\propagate}{\textsc{propagate}}
\newcommand{\newstate}{\textsc{new-state}}
\newcommand{\propstate}{\textsc{prop-state}}
\newcommand{\propcost}{\textsc{prop-cost}}
\newcommand{\cost}{\textsc{cost}\xspace}
\newcommand{\tracepath}{\textsc{trace-path}}
\newcommand{\nulll}{\textsc{null}}

\newcommand{\addvertex}{\textup{add\_vertex}}
\newcommand{\addedge}{\textup{add\_edge}}
\newcommand{\init}{\textup{init}}

\newcommand{\constraints}{h}
\newcommand{\var}{\alpha}

\newcommand\icraVer[2]{\ifthenelse{\isundefined{\extended}}{#1}{#2}}


\def\extended{}

\begin{abstract}
We present a novel analysis of AO-RRT: a tree-based planner for motion planning with kinodynamic constraints, originally described by Hauser and Zhou (AO-X, 2016). AO-RRT 
explores the state-cost space and
has been shown to efficiently obtain high-quality solutions in practice without relying on the
availability of a computationally-intensive two-point boundary-value solver. 
Our main contribution is
an optimality proof for the single-tree version of the algorithm---a variant that was not analyzed before. Our proof only requires a mild and easily-verifiable set of
assumptions on the problem and system: Lipschitz-continuity of the
cost function and the dynamics. In particular, we prove that for any
system satisfying these assumptions, any trajectory having a
piecewise-constant control function and positive clearance from the
obstacles can be approximated arbitrarily well by a trajectory found
by AO-RRT. 
We also discuss practical aspects of AO-RRT and present 
experimental comparisons of variants of the algorithm.
\end{abstract}

\section{Introduction}\label{sec:introduction}

\emph{Motion planning} is a fundamental problem in robotics, concerned with allowing autonomous robots to navigate in complex environments while avoiding collisions with obstacles. The problem is already challenging in the simplified geometric setting, and even more so when considering the kinodynamic constraints that the robot has to satisfy. This work is concerned with the latter setting, and consider the case where the robot's system is specified by differential constraints of the form
\vspace{-.05in}
\begin{equation}
\dot{x}= f(x, u),\quad \text{for } x\in \mathcal{X}, u\in \U,
\label{eq:dyn}
\vspace{-.05in}
\end{equation}
where $\X\subseteq \dR^d$ is the robot's state space, and $\U\subseteq\dR^D$ is the control space, for some $d,D\geq 2$. The objective of motion planning is thus to find a control function $\Upsilon:[0,T]\rightarrow \U$, which induces a \emph{valid} trajectory $\pi:[0,T]\rightarrow \X$, such that (i) Equation~\eqref{eq:dyn} is satisfied, (ii) $\pi$ is contained in the free space $\F\subseteq \X$, and (iii) the motion takes the robot from its initial state $\xinit$ to the goal region $\Xgoal\subseteq \X$.

In \emph{optimal motion planning}, the objective is to find a control function $\Upsilon$ and a trajectory $\pi$ satisfying the constraints (i), (ii), (iii), which also minimize the trajectory cost, specified by 
\vspace{-.05in}
\begin{equation}
\cost(\pi)= \int_0^T g(\pi(t),\Upsilon(t))dt, 
\label{eq:cost}
\vspace{-.05in}
\end{equation}
where $g:\X\times\U \rightarrow \dR_+$ is a cost derivative. Depending on the precise formulation of $g$, $\cost(\pi)$ may represent the distance traversed by the robot,  the energy required to execute the motion, or other metrics.

Almost thirty years of research on motion planning have led to a variety of approaches to tackle  the problem, ranging from computational-geometric algorithms, potential fields, optimization-based methods, and search-based solutions~\cite{KavrakiLaValle08,LaValle06}. To the best of our knowledge, the only approach that is capable of satisfying global optimality guarantees, while still being computationally practical, is sampling-based planning. Sampling-based algorithms capture the connectivity of the free space of the problem via random sampling of states (and sometimes controls) and connecting nearby states, to yield a graph structure.

The celebrated work of Karaman and Frazzoli~\cite{KF11} laid the foundations for optimality in sampling-based motion planning. They introduced several new algorithms and proved mathematically that they converge to the optimal solution as the number of samples generated by the algorithms tends to infinity. This property is termed asymptotic optimality (AO). Many researchers have followed their footsteps, and designed new algorithms, which can be used in various applications~\cite{JSCP15,SH15,GSB15,SolHal16}. 

Unfortunately, the applicability of most of the aforementioned results to optimal planning with kinodynamic constraints remains limited. In particular, the majority of results only apply to the geometric (holonomic) setting of the problem. While a small subset of results do consider the kinodynamic case, they assume the existence of a two point \emph{boundary value problem (BVP) solver}, which given two states $x,x'\in \X$ returns the lowest-cost trajectory connecting them (see, ~\cite{SJP15,SJP_ICRA15,KarFra10,KarFra13,PerETAL12,WebBer13,XieETAL15}). In practice BVP solvers are usually only available for simple robotic systems, and in many cases they are prohibitively costly to use, which limits their applicability.

Recently, there have been sampling-based approaches that do not rely on the existence of a BVP solver~\cite{PapaETAL14}. These methods employ forward propagation instead. Li et al.~\cite{LiWAFR14,LiETAL16} provided an analysis of tree sampling-based planners that perform random propagation from first principles and proposed the \sst algorithm. \sst is in practice computationally efficient and achieves asymptotic near-optimality, which is the property of converging toward a path with bounded suboptimality. True AO properties can be achieved by \sststar, which sacrifices computational efficiency by progressively shrinking a pruning radius parameter. The approach proposed here aims for AO properties and computational efficiency, while avoiding the critical dependence on parameters, such as pruning radii, that are difficult to tune for a variety of motion planning problems.


Most recently, Hauser and Zhou~\cite{HauserZ16} proposed a meta algorithm $\texttt{AO-x}$, which allows to adapt any \emph{well-behaved} non-optimal kinodynamic sampling-based planner, denoted by $\texttt{x}$, into an AO algorithm. This is achieved by substituting the $d$-dimensional state space $\X$ on which the former is run with the $(d+1)$-dimensional space $\Y=\X \times \dR$, where the last coordinate encodes the solution cost. Then, $\texttt{x}$ is iteratively applied to shrinking subsets $\Y_i$ of $\Y$ for $i\in\mathbb{N}_+$, where the maximal value of the last coordinate (representing the cost) is gradually decreased with $i$, and hence the cost of the returned solution. The authors combined their framework with the forward-propagating versions of \rrt~\cite{LaVKuf01} and \est~\cite{HsuKLR02}, to yield \aorrt and \aoest, both of which demonstrated favorable performance over competitors. The observation that the cost induced by a system can be analyzed by augmenting the state space in the above manner was first considered by Pontryagin (see,~\cite{pontryagin1987}).


We follow up on Hauser and Zhou's approach. We augment their work by addressing aspects of the analysis that we believe require more attention, namely what are the precise conditions under which using the augmented-space approach will lead to provably AO solutions. The main issue that we address is the assumption~\cite{HauserZ16} that $\texttt{x}$ is well-behaved, without proving this property for neither \rrt nor for \est. Well behavedness consists of two requirements: (i) \texttt{x} must find a feasible solution eventually within each $\Y_i$---a property corresponding to probabilistic completeness (PC)~\cite{CBHKKLT05}--- and (ii) the cost of the solution found in $\Y_i$ is smaller (with non-negligible probability) than the maximal cost value over~$\Y_i$. Note that requirement (ii) is a particularly strong assumption, essentially requiring \texttt{x} to be ``nearly'' AO, i.e., gradually reducing the cost of the solution when applied to the bounded subspaces $\Y_i$ for $i\in\mathbb{N}_+$.

In this context, it should be noted that some variants of \rrt are not even PC~\cite{KunzS14} (and thus not well behaved). Furthermore, it is not specified for what types of robotic systems~\cite{HauserZ16}, with respect to $\X,f,g$, or problem instances $\F,\xinit,\Xgoal$ this property holds. {Another logical gap that has not been adequately addressed is that the proof focuses on a version of \aorrt which grows multiple trees, and does not seem to directly extend to the single-tree version of \aorrt used in the experiments of that paper.} 


\subsection{Contribution} 

{We present a novel analysis of AO-RRT: a tree-based planner for motion planning with kinodynamic constraints, originally described by Hauser and Zhou~\cite{HauserZ16}.
We focus on a variant that constructs a single tree, rather than multiple trees, embedded in the augmented state space~$\Y$, and which was not analyzed before. We note that this variant  was  used in the experiments in~\cite{HauserZ16}.
The approach does not require a BVP solver and can be viewed as an AO generalization of the non-AO \rrt planner~\cite{LaVKuf01}. }

{Our main contribution is a rigorous optimality proof for the single-tree \aorrt. Our proof only requires an easily-verifiable set of assumptions on the problem and system: we require  Lipschitz-continuity of the cost function and the dynamics. In particular, we prove that for any system satisfying these assumptions, any trajectory having a piecewise-constant control function and positive clearance from obstacles can be approximated arbitrarily well by a trajectory found by \aorrt.  
\icraVer{}{(We also discuss extensions to trajectories whose control function is not necessarily piecewise constant.)} 
Furthermore, we develop explicit bounds on the convergence rate of the algorithm. Our AO proof relies on the theory that we have recently developed for the probabilistic completeness of \rrt~\cite{KSLBH18}.}

{We also discuss practical aspects of \rrttwo, namely node pruning and a hybrid approach that combines the algorithm with other planners, while still maintaining AO. Then we present an experimental comparison of \rrttwo variants with the vanilla \rrt, and \sst for \icraVer{a fixed-wing plane and a rally car.}{both geometric and kinodynamic scenarios.}}

The paper is organized as follows. The \rrttwo algorithm is described in Section~\ref{sec:rrttwo}. Section~\ref{sec:theoretical_prop} proceeds with the theoretical properties of \rrttwo and gives the asymptotic optimality proof. Practical aspects of the algorithm are discussed in Section~\ref{sec:practical_aspects} and experiments are presented in Section~\ref{sec:experiments}. Finally, in Section~\ref{sec:discussion} we discuss further research.

\section{{The single-tree} AO-RRT algorithm}
\label{sec:rrttwo}
{We describe the single-tree \aorrt approach. Henceforth we will refer to this algorithm simply as \rrttwo.} 
Recall that $\X,\F,\U$ denote the state, free, and control spaces, respectively. We assume that $\X$ is compact, and $\F$ is open. 
The \rrttwo algorithm is very similar to the (kinodynamic) \rrt algorithm, based on~\cite{LaVKuf01}.
Whereas \rrt grows a tree embedded in $\X$, \rrttwo (see Algorithm~\ref{alg:rrttwo}) does so in the state-cost space. 
In particular, we define the augmented (state) space $\Y\coloneqq\X\times \dR_+$, 
which is $(d+1)$-dimensional, where the additional coordinate represents the cost of the (non-augmented) state. That is, a point $y\in \Y$ can be viewed as a pair $y=(x,c)$, where 
$x\in \X$ and $c\geq 0$ represents the cost of the trajectory from $\xinit$ to $x$ over the tree $\T(\Y)$. Given a point $y\in \Y$ we use the notation $x(y),c(y)$ to represent its component of $\X$ and cost, respectively.

The \rrttwo algorithm has the following 
inputs: In addition to an initial start state $\xinit$, goal region $\X_{\text{goal}}$, number of iterations $k$, maximal total duration for propagation $\Tprop$, and control space $\U$, which \rrt accepts, \rrttwo also accepts a maximal cost $\cmax$. 
See Section~\ref{sec:practical_aspects} for more information on how to choose $\cmax$.

\rrttwo constructs a tree $\T(\Y)$, embedded in $\Y$ and rooted in $\yinit=(\xinit,0)$, by performing $k$ iterations of the following form. In each iteration, it generates a random sample $\yrand$ in $\Y$, by randomly sampling $\X$ and the cost space $[0,\cmax]$ (lines~3-4). In addition a random control $\urand$ and duration $\trand$ are generated by calling the routine $\sample$ (lines~5-6). For a given set $S$, the procedure $\sample(S)$ produces a sample uniformly and randomly from $S$.

Next, the nearest neighbor $\ynear$ of $\yrand$ in $\T(\Y)$ is retrieved (line~7). We emphasize that this operation is performed in the $(d+1)$-dimensional space $\Y$ using a suitable distance metric such as the Euclidean metric in the augmented space (see Section~\ref{sec:practical_aspects}). 
Then, in line~8, the algorithm uses a forward propagation approach (using \propagate) from $\ynear$ to generate a new state $\ynew$: the random control input $\urand$ is applied for time duration $\trand$ from $x(\ynear)$ reaching a new state $\xnew\in \X$ through a trajectory $\pinew$. The state $x(\ynear)$ is then coupled with the cost of executing $\pinew$ together with $c(\ynew)$ (line~9). Mathematically, for $x\in \X,u\in \U,t>0$, we have that 
\[\propagate(x, u, t)\coloneqq \int_0^{t}f(x(t),u)\dt.\]

Finally, \collisionfree($\pinew$) checks whether the trajectory reaching $\ynew$ from $\ynear$ using the control $\urand$ and duration $\trand$ is collision free. This operation is known as \emph{local planning}, and is typically achieved by densely sampling the trajectory and applying a dedicated collision detection mechanism~\cite{CRCbookChap39}. If indeed the trajectory is collision free, $\ynew$ is added as a vertex to the tree and is connected by an edge from $\ynear$ (lines 10-12). The trajectory $\pinew$ is also added to the edge. If $\ynew$ is in the goal region and its cost is the smallest encountered so far, then $\ymin$ is substituted with this point (lines~13,14). Finally, a lowest-cost trajectory (if exists) is returned in line~15.
Note that the algorithm maintains the lowest-cost trajectory discovered so far by keeping track of the last vertex $\ymin$ on such a trajectory.

\section{Theoretical properties of AO-RRT}
\label{sec:theoretical_prop}
We 
spell out
the assumptions that we make with respect to the system and the cost function, and state our main theorem.
Then, in Section~\ref{subsec:augsystem},
we describe the problem in the augmented space $\Y$,  define the augmented system $F$, and study its properties.
We then leverage this in the proof of the main theorem in Section~\ref{subsec:proof}. \icraVer{}{In Section~\ref{sec:optimal_control} we discuss the extension of the theorem to trajectories not necessarily having piecewise-constant control functions.}

Throughout this section we use the following notations. 
For simplicity, in our proofs we use the standard Euclidean norm, denoted by $\|\cdot\|$.
We note, however, that all proofs can be generalized to work with the weighted Euclidean norm.
Given a set $S\subseteq \dR^{d'}$, for some $d'>0$, we denote by $|S|$ its Lebesgue measure.
For a given point $y\in \dR^{d'}$, and a radius $r>0$, we use $\B_r^{d'}(y)$ to denote the $d'$-dimensional Euclidean ball of radius $r$ centered at $y$. 

We make the following assumption concerning $f$ (Eq.~\eqref{eq:dyn}):
\begin{assumption}[Lipschitz continuity of the system]\label{assume:f_lip}
The system $f$ is Lipschitz continuous for both
of its arguments.
That is,
$\exists K^f_u, K^f_x > 0 $ s.t. $\forall\ x_0,x_1\in \mathcal{X}, \forall u_0,u_1\in \U$:
\begin{align*}
\| f(x_0, u_0)- f(x_0, u_1)\| \leq K^f_u\| u_0-u_1\|,\\
\| f(x_0, u_0)- f(x_1, u_0)\| \leq K^f_x\| x_0-x_1\|.\end{align*}
\end{assumption}

We make the following assumption concerning $g$ (Eq.~\eqref{eq:cost}):
\begin{assumption}[Lipschitz continuity of the cost]\label{assume:g_lip}
The cost derivative $g$ is Lipschitz continuous for both
of its arguments. That is,
$\exists K^g_u, K^g_x > 0 $ s.t. $\forall\ x_0,x_1\in \mathcal{X},\forall u_0,u_1\in \U$:
\begin{align*}\| g(x_0, u_0)- g(x_0, u_1)\| &\leq K^g_u\| u_0-u_1\|,\\
\| g(x_0, u_0)- g(x_1, u_0)\| &\leq K^g_x\| x_0-x_1\|.\end{align*}
\end{assumption}

\begin{definition}
	A piecewise constant control function $\overline{\Upsilon}$  with resolution $\Delta t$ is the concatenation of constant control functions $\bar{\Upsilon}_i : [0, \Delta t ] \rightarrow u_i$, where $u_i\in \U$, and $1\leq i\leq k$, for some $k\in \mathbb{N}_{>0}$. \label{def:picewise}
\end{definition}

From this point on, when we say a \emph{valid trajectory} we mean a valid trajectory as described in Section~\ref{sec:introduction}, with the extra proviso that the control function is piecewise constant. 

\begin{definition}
    Let $\pi$ be a valid trajectory, and let $T$ be its duration. We define the clearance of $\pi$ to be the maximal value $\delta>0$ such that 
    \[\bigcup_{t\in [0,T]}\B^d_{\delta}(\pi(t))\subset \F\text{ and } \B_\delta^d(\pi(T))\subset \Xgoal.\]
    We say that a trajectory is \emph{robust} if its clearance is positive. 
\end{definition}

We arrive to our main contribution that establishes the rate of convergence of \rrttwo.

\begin{theorem}
Assume that Assumptions~\ref{assume:f_lip}, \ref{assume:g_lip} hold and fix \mbox{$\eps\in(0,1)$}. Denote by $\pi_k$ the solution obtained by \rrttwo after $k$ iterations. For every robust trajectory $\pi$ having a piecewise-constant control function there exist a finite $k_0\in\mathbb{N},a>0,b>0$, such that for every $k>k_0$ it holds~that 
\begin{align*}\Pr[\cost(\pi_k)>(1+\eps)\cost (\pi)]\leq ae^{-bk}.\end{align*}
\label{thm:main}
\end{theorem}

\begin{algorithm}[b!]
	\caption{${\tt {AO-RRT}}(\xinit, \Xgoal, k, \Tprop, \U, \cmax)$}
	\begin{algorithmic}[1]
		\State{$\yinit\gets (\xinit,0); \T(\Y).\init(\yinit); \ymin=(\nulll,\infty)$}
		\For {$i = 1 \text{ to } k$}
		\State $\xrand\gets \sample(\X)$ \Comment{sample state}
		\State $\crand \gets \sample([0,\cmax])$ \Comment{sample cost}
		\State $\trand \gets \sample([0, \Tprop])$ \Comment{sample duration}
		\State $\urand \gets \sample(\U)$ \Comment{sample control}
		\State $\ynear\gets \nearest(\yrand=(\xrand,\crand),\T(\Y))$  
		\State $(\xnew,\pinew) \gets \propagate(x(\ynear), \urand, \trand)$
		\State $\cnew \gets c(\ynear) + \cost(\pinew)$
		\If {\collisionfree($\pinew$)} 
		\State{$\T(\Y).\addvertex(\ynew=(\xnew,\cnew))$}
		\State{$\T(\Y).\addedge(\ynear, \ynew, \pinew)$}
		\If {$x(\ynew )\in \Xgoal$ and $c(\ynew)<c(\ymin)$}
		\State $\ymin \gets\ynew$ 
		\EndIf
		\EndIf
		\EndFor	
		\State 	\Return $\tracepath(\T(\Y),\ymin)$
	\end{algorithmic}
	\label{alg:rrttwo}
\end{algorithm}

\vspace{-30pt}
\subsection{Properties of the augmented system}
\label{subsec:augsystem}
It would be convenient to view the problem of optimal planning with respect to $f,g$, as a feasible motion planning for an augmented system $F$, which is defined as follows. The \emph{augmented} system $F$ encompasses both types of transitions in $f$ and $g$, respectively. The control space for this system is simply $\U$, and its state space is $\Y=\X\times \dR_+$. Formally, 
\begin{equation}
\dot{y}=(\dot{x},\dot{c})=F(y, u)=(f(x,u),g(x,u)),
\label{eq:dyn_ext}
\end{equation} 
for $y=(x,c)$, where $x\in \mathcal{X}, c\in \mathbb{R}_+, u\in \U$.  

We have the following claim with respect to $F$: 
\begin{claim}\label{claim:F_lip}
Under Assumptions~\ref{assume:f_lip},\ref{assume:g_lip}, the augmented system $F$ is Lipschitz continuous for both
of its arguments. That is,
$\exists K_u, K_x > 0 $ s.t. $\forall\ y_0,y_1\in \mathcal{Y},u_0,u_1\in \U$:
\begin{align*}
\| F(y_0, u_0)- F(y_0, u_1)\| \leq K_u\| u_0-u_1\|,\\
\| F(y_0, u_0)- F(y_1, u_0)\| \leq K_x\| y_0-y_1\|.
\end{align*}
\end{claim}
\begin{proof} It follows that
{\small\begin{align*}
 \| F(& y_0, u_0) - F(y_0, u_1)\| \\ &= \sqrt{\| f(x_0, u_0)- f(x_0, u_1)\|^2 + \| g(x_0, u_0)- g(x_0, u_1)\|^2} \\
&\leq \sqrt{(K^f_u\| u_0-u_1\|)^2 + \left(K^g_u\| u_0-u_1\|\right)^2} \\
&= \sqrt{(K^f_u)^2 + (K^g_u)^2}\cdot \| u_0-u_1 \|= K_u \cdot \| u_0-u_1 \|,
\end{align*}}for $K_u\eqq \sqrt{(K^f_u)^2 + (K^g_u)^2}$.

The second inequality requires an additional transition since
$\| F(y_0, u_0) - F(y_1, u_0)\|\leq K_x\cdot \| x_0-x_1 \|$.  It remains to use the fact that $\|x_0-x_1\| \leq  \|y_0-y_1\|$.
\end{proof}

We can think of \rrttwo planning for the system $f$ with cost $g$, state space $\X$ and control space $\U$, as the standard \rrt operating over the system $F$, state space $\Y$, and control space $\U$. Lines 8-9 in Algorithm~\ref{alg:rrttwo}
are identical to propagating with $F$. This equivalence allows to exploit useful properties of \rrt recently developed in~\cite{KSLBH18}. 

\subsection{Proof of Theorem~\ref{thm:main}}
\label{subsec:proof}
We first provide an outline of the proof.
Fix $\eps\in (0,1)$ and let $\pi_\delta$ be a robust trajectory whose clearance is $\delta>0$. The clearance is with respect to both distance from the obstacles, and from the boundary of the goal region. 
Let $c_\delta \coloneqq  \cost(\pi_\delta)$.
We draw $\pi_\delta$ in the $(d+1)$-dimensional space $\Y$, such that the new trajectory $\pi^{\Y}_\delta$ begins in $\yinit = (\xinit, 0)$ and ends in $\ygoal = (\xgoal, c_\delta)$, where $\xgoal\in \Xgoal$.
Next, similarly to~\cite{KSLBH18}, we place a constant number of balls of radius $r=\min(\eps c_\delta, \delta)$ along the trajectory $\pi^{\Y}_\delta$.  The balls are constructed in a manner that guarantees that each such transition is collision free.
Then we show that with high probability \rrttwo will visit all such balls as the number of samples $k$ tends to infinity, by transitioning from one ball to the next incrementally. Reaching the last ball, centered at $\ygoal$, implies that \rrttwo will find a solution whose cost is at most $c_\delta+r \leq c_\delta + \eps c_\delta = (1+\eps)c_\delta$, since by definition any trajectory in $\Y$ that terminates in $\B_r^{d+1}(\ygoal)$ must have a cost (which is its $(d+1)$th coordinate) of at most $c_\delta+r$.

To achieve this, we first adapt with minor changes the following two lemmatta from~\cite{KSLBH18} to the setting of \rrttwo. Lemma~\ref{lem:prop_bound} shows that there exists a constant $\tau\leq\Tprop$ such that if we place the centers of the balls along $\pi^{\Y}_\delta$ where the duration between two consecutive centers is $\tau$ then the probability for successfully propagating from one ball to the next is positive (assuming that  the propagation duration and the control input are chosen uniformly at random).
We note that it follows from~\cite{KSLBH18} that $\tau$ can be chosen such that 
$\pi^{\Y}_\delta$ can be divided into sub-trajectories of duration $\tau$, where the control function is fixed during each sub-trajectory.

\begin{lemma}
There exists $\tau\leq\Tprop$ 
for which the following holds: Let $\pi$ be a trajectory for $F$ with clearance $\delta>0$ and a control function that is fixed during the interval $[0,\tau]$.  Let $r\leq\delta$.
Let $\trand$ be a random duration sampled uniformly from $[0,\Tprop]$, and $\urand$ a uniformly sampled control input from $\U$.
Suppose that the propagation step of \rrttwo begins at state $\ynear \in \B^{d+1}_r(\pi(0))$ and ends in $\ynew$ (lines 8,9 in Algorithm~\ref{alg:rrttwo}). Then
	    \[p_{\textup{prop}}\eqq \Pr\left[\ynew \in \B^{d+1}_{2r/5} (\pi(\tau))\right]>0.\]
	    \vspace{-15pt}
	    \label{lem:prop_bound}
\end{lemma}
Lemma~\ref{lem:nn_prob} shows that the probability that the nearest neighbor of a random sample $\yrand$ lies in a specific ball is positive, when  $\yrand$ is sampled uniformly at random from~$\Y$.
\begin{lemma}
    Let $y\in \Y$ be such that $\B^{d+1}_r(y)\subset \F_\Y\eqq \F\times \dR_+$.
    Suppose that there exists an \rrttwo vertex $v\in \B^{d+1}_{2r/5}(y)$.
    Let $\ynear$ denote the nearest neighbor of $\yrand$ among all \rrttwo vertices.
    Then
        \[p_{\textup{near}}\eqq \Pr\left[\ynear\in \B^{d+1}_{r}(y)\right]\geq |\B^{d+1}_{r/5}|/|\Y|. \]
        \vspace{-15pt}
    \label{lem:nn_prob}
    
\end{lemma}

Note that both probabilities $p_{\textup{prop}}, p_{\textup{near}}$ are independent of the number $k$ of iterations of the algorithm. Next, we will 
place $m+1$ balls of radius $r=\min(\eps c_\delta, \delta)$ centered at states $y_0=\yinit, y_1,\ldots , y_m = \ygoal$ along the trajectory $\pi^{\Y}_\delta$. See Figure~\ref{fig:covering_balls} for an illustration.

\begin{figure}[t]
    \centering
    \vspace{10pt}
    \includegraphics[width=\columnwidth]{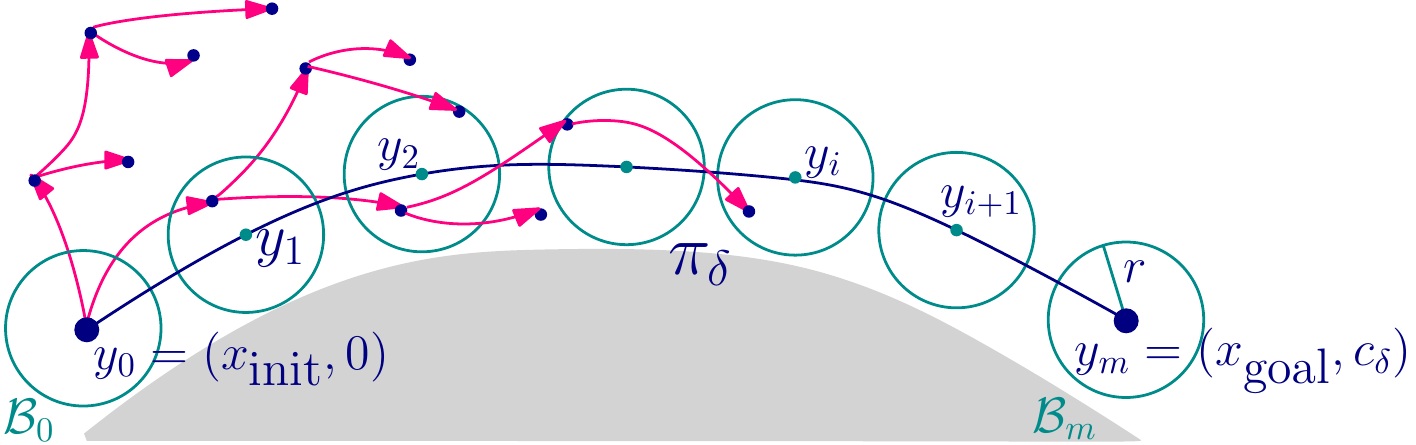}
    \vspace{-20pt}
    \caption{Illustration of the proof of Theorem~\ref{thm:main}.}
    \label{fig:covering_balls}
    \vspace{-15pt}
\end{figure}

Denote by $t_\delta$ the duration of $\pi_\delta$. We determine the sequence of points $y_0,\ldots, y_m$ in the following manner: Choose a set of durations $t_0 = 0, t_1, t_2, \ldots  , t_m = t_\delta$, such that the difference between every two consecutive
ones is $\tau$ (see Lemma~\ref{lem:prop_bound}). That is, let $y_0 = \pi^{\Y}_\delta(t_0), y_1 = \pi^{\Y}_\delta(t_1 ), \ldots , y_m = \pi^{\Y}_\delta(t_\delta)$
be states along the path $\pi^{\Y}_\delta$ that are 
obtained after duration $t_0, t_1, \ldots , t_m$, respectively.
Obviously, $m = t_\delta/\tau$ is some constant independent
of the number of samples.

Suppose that there exists an \rrttwo vertex
$v \in \B^{d+1}_{2r/5} (y_i) \subset \B^{d+1}_{r} (y_i)$. 
We shall bound the probability $p$
that in the next iteration the \rrttwo tree will extend from a 
vertex in $\B^{d+1}_{r}(y_i)$, given that a vertex in $B^{d+1}_{2r/5} (y_i)$ exists, and that the propagation step will add a vertex to $\B^{d+1}_{2r/5} (y_{i+1})$.
That is, $p$ is the probability that in the next iteration both
$y_\text{near} \in \B^{d+1}_{r} (y_i)$ and $y_\text{new}\in \B^{d+1}_{2r/5} (y_{i+1})$. 
From Lemma~\ref{lem:nn_prob} we 
have that the probability that $y_\text{near}$ lies in $\B^{d+1}_{r} (y_i)$, given that
there exists an RRT vertex in $\B^{d+1}_{2r/5} (y_i)$, is at least 
$p_{\textup{near}}$. Next, we wish to sample duration $\trand$ and control $\urand$ such that a random propagation from $\ynear$ will yield $y_\text{new}\in \B^{d+1}_{2r/5} (y_{i+1})$. According to Lemma~\ref{lem:prop_bound}, the probability for this to occur is at least $p_{\textup{prop}}$. 
Thus, jointly the probability that $\ynear$ falls in $\B^{d+1}_{r} (y_i)$ and of sampling the correct propagation
duration and control is at least $p=p_{\textup{near}}\cdot p_{\textup{prop}}>0$. As we mentioned earlier, this value is also independent of the number of iterations.

It remains to bound the probability of having $m$ successful such steps. This process can be described as $k$ Bernoulli trials with success probability $p$. The  planning problem can be solved after $m$ successful outcomes, where the $i$th outcome adds an \rrttwo vertex in $\B^{d+1}_{2r/5}(y_i)$. Let $X_k$ denote the number of successes in $k$ trials. 
As in~\cite{KSLBH18}, we have that 
\[\Pr[X_k < m] = \sum_{i=0}^{m-1}{\binom{k}{i}p^i(1-p)^{k-i}}\nonumber\leq ae^{-bk},\]
where $a$ and $b$ are  positive constants. This concludes the proof of Theorem~\ref{thm:main}.
\icraVer{{In the full version of the paper~\cite{kleinbort2019rrt20} we show that Theorem~\ref{thm:main} is not limited to trajectories with piecewise-constant control functions.}
}{
\subsection{Beyond piecewise-constant control}\label{sec:optimal_control}
Theorem~\ref{thm:main} argues that for any robust trajectory $\pi$ having a  piecewise-constant control function, with high probability \rrttwo will find a trajectory 
whose control function is piecewise constant and whose cost is at most $(1+\eps)\cost(\pi)$, where $\eps>0$ is a constant.

Next, we show that this theorem is not limited to trajectories with piecewise-constant control functions. Let $\pi^*$ be the optimal trajectory with respect to $\cost$, with control function $u^*$ and duration $T^*$. Notice that $\pi^*$ is not necessarily robust, $u^*$ is not necessarily piecewise constant. Nevertheless we show that such $\pi^*$ can be approximated arbitrarily well, with respect to $\alpha>0$, using a trajectory $\pi_\alpha$, which has piecewise-constant control. This implies that to achieve a cost close to $\cost(\pi^*)$ it suffices to apply Theorem~\ref{thm:main} and get close to~$\pi_\alpha$.

The following proposition states that for any optimal solution $(T^*,\pi^*,u^*)$, which satisfies certain assumptions, (1) for any $\alpha>0$, there exists a robust solution $(T_{u_\alpha},\pi_{u_\alpha}, u_\alpha)$, where the control function $u_\alpha$ is in $L^{\infty}$, i.e., bounded, but not necessarily piecewise constant, and the cost of $\pi_{u_\alpha}$ is at most $1+\alpha$ times the cost of $\pi^*$. This statement is then used to prove part (2) of the proposition, which asserts that a similar result holds even when $u_\alpha$ is piecewise constant. In the following, it would be convenient to represent the free space as 
$\F = \{ x \in \mathbb{R}^d | \constraints(x) \le 0 \}$, where $\constraints(x)$ can be interpreted as the negative value of the clearance of $x$. 

\begin{proposition} \label{ref:PropDensity}
Assume that $\F = \{ x \in \mathbb{R}^d | \constraints(x) \le 0 \}$, where $\constraints : \mathbb{R}^d \rightarrow \mathbb{R}$ is of class $C^1$, and that $\U = \mathbb{R}^D$. Assume also that the dynamics $f$ and the cost $\cost$ are $C^1$ functions, and that there exists an optimal strategy $(T^*,\pi^*,u^*)$ that has a unique extremal which is moreover normal. Then, the following holds:
\begin{enumerate}
\item For every $\var > 0$, there exist $\delta_{\var} > 0$ and a control $u_{\var} \in L^{\infty}([0,T_{\var}],\U)$ such that the related trajectory $\pi_{u_{\var}}$ is defined in $[0,T_{\var}]$ and satisfies
\begin{align*}
&|\cost(\pi^*) - \cost(\pi_{u_{\var}})| < \var,  &\pi_{u_{\var}}(T_{\var}) = \pi^*(T^*),
\end{align*}
 and $\constraints(\pi_{u_{\var}}(t)) = -\delta_{\var} < 0,$ for $t \in [0,T_{\var}]$.
 \item For every $\var > 0$, there exist $\delta_{\var} > 0$ and a piecewise-constant control $u_{\var}$ defined in $[0,T_{\var}]$ such that the related trajectory $\pi_{u_{\var}}$ is defined in $[0,T_{\var}]$ and satisfies
\begin{align*}
&|\cost(\pi^*) - \cost(\pi_{u_{\var}})| < \var, \\ &\| \pi_{u_{\var}}(T_{\var}) - \pi^*(T^*) \| < \var,
\end{align*}
 and $\constraints(\pi_{u_{\var}}(t)) = -\delta_{\var} < 0,$ for $t \in [0,T_{\var}]$.
\end{enumerate}
\end{proposition}

For this, we recall that, given a feasible strategy $(T,\pi,u)$ for the motion planning problem, a related {\em extremal} $(T,\pi,u,p,\mu,p_0)$, where $p^0 \le 0$ is constant, $p : [0,T] \rightarrow \mathbb{R}^d$ is an absolutely continuous function, and $\mu$ is a non decreasing function of bounded variation, is by definition a quantity satisfying the {\em Pontryagin Maximum Principle} \cite{pontryagin1987,dmitruk2009development}, i.e., the Pontryagin adjoint equations, maximality and transmission conditions (see \cite{dmitruk2009development} for precise definitions). The Pontryagin Maximum Principle is a necessary condition for optimality, therefore, to any optimal solution $(T^*,\pi^*,u^*)$ it is associated a non-trivial extremal $(T^*,\pi^*,u^*,p^*,\mu^*,p^*_0)$. An important class of extremals are the so-called normal extremals, that by definition satisfy $p^0 \neq 0$.

For what concerns Proposition \ref{ref:PropDensity}, the assumption on the existence of a unique normal extremal requires some (informal) comments. Normal extremals naturally exist for optimal control problems and are often unique (see, e.g., \cite{chitour2008singular,trelat2000some}). Their uniqueness is related to the regularity of solutions to the Hamilton-Jacobi-Bellman equation: a smooth solution provide an (at least locally) unique normal extremal (see, e.g., \cite{trelat2000some}). Since the regularity of the solutions to the Hamilton-Jacobi-Bellman equation are related to the regularity of the data, enough regular dynamics, cost and scenario (i.e., at least $C^1$) provide the existence of unique normal extremals.

The proof of Proposition \ref{ref:PropDensity} makes use of the surjective form of the {\em Implicit Function Theorem} in infinite dimensional Banach spaces (see, e.g., \cite{antoine1990etude}). The assumption on the existence of a unique normal extremal will be crucial to apply the theorem to our framework. Below, we provide a sketch-of-proof considering fixed final time $T$ (for free final time $T$, the proof goes similarly with slight modifications, see also \cite[pp. 310--314]{lee1967foundations}). \\

\noindent {\em Sketch-of-proof of Proposition \ref{ref:PropDensity}:} 
Consider fixed final time $T$ (therefore, with the notation in Proposition \ref{ref:PropDensity}, $T = T^* = T_{\alpha}$) and let us introduce the following new family of constraints:
\begin{equation} \label{constraints}
\constraints_{\delta}(x) := \constraints(x) + \delta
\end{equation}
where $\delta \in \mathbb{R}$. Since $\pi^*$ is defined in $[0,T]$, it is easy to prove that, by multiplying the dynamics $f$ by {\em smooth cut-off functions} (see, e.g., \cite{lee2001introduction}) around $\pi^*$, for every control $u \in L^{\infty}([0,T],\U)$, the related trajectory $\pi_u$ is defined in the whole interval $[0,T]$ (see, e.g., \cite{trelat2000some}). Therefore, the following infinite-dimensional, parameter-dependent {\em End-Point Mapping} is correctly defined
\begin{align*}E : [0,1] \times L^{\infty}([0,T],\U) \rightarrow \mathbb{R}^d \times C^0([0,T],\mathbb{R})\\
(\delta,u) \mapsto \Big( \pi_u(T) - \pi^*(T), \constraints_{\delta}(\pi_u(\cdot)) - \constraints(\pi^*(\cdot)) \Big).\end{align*}
Moreover, by the differentiability of $\pi_u$ with respect to $u$ (see, e.g., \cite{trelat2000some}), the mapping $E$ is $C^1$. Remark that to obtain such differentiability properties we need to ask that $f$, $\constraints$ and $\cost$ are $C^1$, which is among our first assumptions (the $C^1$ regularity of $\cost$ is required for the existence of any Pontryagin extremal in the smooth case, see, e.g., \cite{dmitruk2009development}).

At this step, we make use of the surjective form of the Implicit Function Theorem in infinite dimensional Banach spaces applied to the End-Point Mapping $E$ above. The theorem can be applied because we assume the existence of a unique and moreover normal extremal related to $(T,\pi^*,u^*)$, which implies that the differential with respect to $u$ of $E$ at $(0,u^*)$ is surjective. From this, by adapting the framework considered in \cite{haberkorn2011convergence,bonalli2019continuity} (that is, replacing control constraints with pure state constraints), one proves that there exist $r > 0$ and a continuous mapping $\varphi : [0,r) \rightarrow L^{\infty}([0,T],\mathbb{R}^D)$ (with respect to the topology of $L^{\infty}$, see, e.g., \cite{brezis2010functional}) such that $\varphi(0) = u^*$ and $E(\delta,\varphi(\delta)) = 0$ for every $\delta \in [0,r)^k$. In other words:
 \begin{equation} \label{controls}
\begin{split}
&\forall \ \delta \in [0,r): \ \pi_{\varphi(\delta)}(T) = \pi^*(T),\\
&\constraints_{\delta}(\pi_{\varphi(\delta)}(t)) = \constraints(\pi^*(t)) \le 0, \ t \in [0,T].
\end{split}
\end{equation}
Now, by denoting $u_{\delta} := \varphi(\delta) \in L^{\infty}([0,T],\mathbb{R}^D)$, the continuity of $\varphi$ (with respect to $\delta$), of $\pi$ (with respect to $u$) and of $\cost$ (with respect to $\pi$) under appropriate topologies gives that, for $\var > 0$ there exists $\delta_{\var} \in (0,r)^k$ such that
$$
|\cost(\pi^*) - \cost(\pi_{u_{\delta_{\var}}})| < \var,
$$
which together with \eqref{constraints} and \eqref{controls} provides the first claim.

To obtain the second claim, we just need to approximate controls $u_{\delta}$ above with piecewise constant controls. Similarly to above, since $\cost$ is continuous with respect to the topology of $L^{\infty}$, if we fix $\var > 0$, there exists $s_{\var} > 0$ such that $|\cost(\pi^*) - \cost(\pi_u)| < \var$ for every control $u \in L^{\infty}([0,T],\mathbb{R}^D)$ for which $\| u^* - u \|_{L^{\infty}} < s_{\var}$. Now, thanks to the continuity in $L^{\infty}$ of the mapping $\varphi$ and the fact that $\varphi(0) = u^*$, there exists $\delta_{\var} \in (0,r)^k$ such that $\| u^* - u_{\delta_{\var}} \| < s_{\var}/4$. Now, recall that the set of piecewise constant 
functions is dense in $L^{\infty}([0,T],\mathbb{R}^D)$. This means that there exists a piecewise-constant control $u_{\var}$ 
such that $\| u_{\delta_{\var}} - u_{\var} \|_{L^{\infty}} < s_{\var}/4$. Importantly, up to reducing the value of $s_{\var} > 0$, the continuity of trajectories $\pi$ with respect to $u$ (in the topology of $L^{\infty}$) gives that $\pi_{u_{\var}}$ is defined in the whole interval $[0,T]$ (use smooth cut-off functions as above) and that the following holds by \eqref{controls} and the continuity of $\constraints$:
\begin{align*}\| \pi_{u_{\var}}(T) - \pi^*(T) \| < \var, \ \constraints(\pi_{u_{\var}}(t)) = -\bar \delta_{\var} < 0, \ t \in [0,T]\end{align*}
for a given $\bar \delta_{\var} > 0$. Since $\| u^* - u_{\var} \|_{L^{\infty}} \le \| u^* - u_{\delta_{\var}} \|_{L^{\infty}} + \| u_{\delta_{\var}} - u_{\var} \|_{L^{\infty}} < s_{\var}/4 + s_{\var}/4 = s_{\var}/2 < s_{\var}$, from above
$|\cost(\pi^*) - \cost(\pi_{u_{\var}})| < \var$
and the conclusion follows. 
}

\section{Practical aspects of AO-RRT}
\label{sec:practical_aspects}
We discuss several approaches to {potentially} speed up the performance of \rrttwo in practice, while retaining its AO property.
\vspace{5pt}

\noindent\textbf{Cost sampling.} \rrttwo samples $(d+1)$-dimensional points from the augmented space $\Y$ by randomly sampling $\X$ and the cost space $[0, c_\text{max}]$, where $c_\text{max}$ provides an upper bound on the maximal cost of the solution.  Setting $c_\text{max}$ to be much larger than the cost of existing tree vertices may bias the $\nearest$ procedure towards selecting vertices with high cost, which may affect the time to find an initial solution. Thus, we propose to set $c_\text{max}$ to be the maximal cost among the tree nodes, until an initial solution is found. Then, we can fix $c_\text{max}$ to be the cost of the solution.

\vspace{5pt}

\noindent\textbf{Augmented-space metric.} In some applications the coordinates of the $\X$-component and the cost component in $\Y$ may be on different scales, which can bias $\nearest$ procedure towards either the cost or the $\X$ component. This in turn may affect the behaviour of the algorithm and its convergence rate. Thus, we propose to use a Weighted Euclidean metric for $\nearest$, defined as
\begin{equation} \label{eq:dist_func}{\textsc {dist}}(y_a,y_b)\coloneqq\sqrt{w_x \|x_a - x_b\|^2 + w_c |c_a - c_b|^2 }, 
\end{equation}
where $y_a = (x_a, c_a), y_b = (x_b, c_b) \in \Y$.
To avoid biasing, $w_x,w_c$ should be chosen such that the maximal possible squared distance between the $\X$ components and the maximal possible squared distance between costs would be of the same order. Note that this weighted version can be viewed as using an unweighted version on an augmented space $\Y'$ in which the cost coordinate has been rescaled. Thus, the theoretical analysis presented in the previous section holds for the weighted version as-is.
\vspace{5pt}

\noindent\textbf{Node pruning.} After a solution of some cost $c>0$ is found, existing tree nodes whose  cost-to-come is greater than $c$ cannot participate in the returned solution, or in a solution of better cost. Such vertices can therefore be removed from the tree. We emphasize that the proof of the previous section still applies to this setting, as after pruning, the probability to grow the tree from a certain node whose cost-to-come value is at most $c$ only increases. 
%
%
\vspace{5pt}

\noindent\textbf{Hybrid planning.} 
As the performance of sampling-based planners varies from one scenario to another, we propose a hybrid approach \hybrrttwo, combining \rrttwo with other planners. This approach may perform better in scenarios where \rrttwo struggles to find a solution, and effectively guide the planning and expedite the convergence towards the optimum. 
\hybrrttwo combines \rrttwo with an additional tree planner, termed $\texttt{PLN}$, while operating in the augmented space $\Y$. It extends the constructed tree by alternating between \rrttwo and $\texttt{PLN}$. Each node added to the tree is assigned with a cost value, as in \rrttwo.

\hybrrttwo is AO since by applying \rrttwo every other iteration we still have a positive probability $p' = p/2$ for a successful transition from $\B^{d+1}_i$ to $\B^{d+1}_{i+1}$, for every $i$.
Moreover, the addition of tree nodes due to the steps of the other planner $\texttt{PLN}$ does not affect the transition probability $p'$.
We note that $\texttt{PLN}$ is not required to be AO, nor is it assumed to be PC. This is in the spirit of Multi-Heuristic~A*~\cite{AineSNHL15}, where multiple inadmissible heuristic functions are used simultaneously with a single consistent heuristic to preserve guarantees on completeness of the search.

\ignore{We propose two concrete examples of the hybrid approach.
The first, \hybrrttwostride, 
alternates between \rrttwo and \stride~\cite{GipsonMK13}.
\stride uses a data structure that enables it
to produce density estimates in the full state
space.  More precisely, it samples a configuration $s$, biased towards relatively unexplored
areas of the state space. The tree is then grown from $s$, if possible. \stride was not shown to be AO.
\hybrrttwostride maintains, as \stride does, a data structure for states in the augmented space $\Y$ and alternates between the two methods for choosing the node $\ynear$ to grow the tree from.
Once the node is chosen, the algorithm proceeds as \rrttwo (line 7 in Algorithm~\ref{alg:rrttwo} is replaced with the \stride method of choosing $\ynear$).
Similarly, we define \hybrrttwoest, which alternates between \rrttwo and \est~\cite{HsuKLR02}. The latter attempts to detect and expand from less explored areas of the space through the use of a grid imposed on a projection of the state space. }

\vspace{-5pt}
\section{Experimental results}
\label{sec:experiments}
\icraVer{%
{We present an experimental evaluation of the performance
of \rrttwo.
%
Our experiments were conducted on an Intel(R) Xeon(R) CPU E5-1660 v3\@3.00GHz with 32GB of memory. 
We set the optimization objective to be the duration of the trajectory. 
Throughout the experiments we assign $w_x = w_c$ in the distance metric 
(Eq.~\eqref{eq:dist_func}) for all \aorrt variants.
Experiments with different weighting schemes or with additional scenarios
are provided in~\cite{kleinbort2019rrt20}.}

{We compare the algorithms within the \aorrt framework with \rrt~\cite{LaVKuf01} and \sst~\cite{LiETAL16}, which have weaker guarantees. The \aorrt variants that we used are:
(i) \aorrtrebuilding---the algorithm analyzed in~\cite{HauserZ16}, (ii) \aorrtnopruning---an implementation of~Algorithm~\ref{alg:rrttwo}, (iii) \aorrtprune, which performs node pruning, and (iv) \hybrrttwostride. 
The latter is a hybrid planner, as described in Section~\ref{sec:practical_aspects}, combining \rrttwo and \stride~\cite{GipsonMK13}.
We mention that \stride was originally defined for geometric settings and was not shown to be AO. \stride uses a data structure that enables it
to produce density estimates in the full state
space. More precisely, it samples a configuration $s$, biased towards relatively unexplored areas of the state space. The tree is then grown from $s$, if possible. \hybrrttwostride maintains, as \stride does, a data structure for states in the augmented space $\Y$ and alternates between the two methods for choosing the node $\ynear$ to grow the tree from.
Once the node is chosen, the algorithm proceeds as \rrttwo does.}

{For each planner we report on both the success rate and the minimum solution cost averaged over all successful runs, displaying values within one standard deviation of the mean. Each result is averaged over 50~runs. Note that since we plot the average cost over all successful runs, we may observe for a certain planner an increase in the average cost. This is only possible if the success rate increases as well.}

\begin{figure}
    \centering
    \vspace{2pt}
    \includegraphics[width=0.22\textwidth]{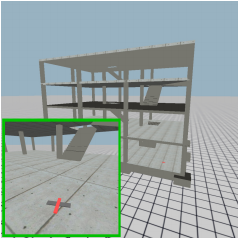} 
   \includegraphics[width=0.21\textwidth]{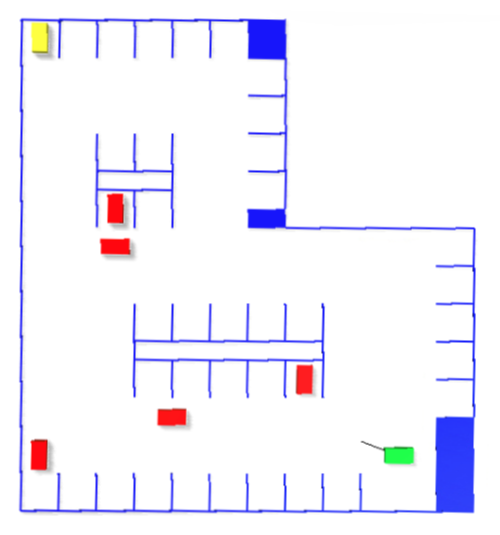}
    \vspace{-5pt}
    \caption{Scenarios: Fixed-wing airplane in a building (left), and rally car in a parking lot (right).}
    \label{fig:scenarios}
\end{figure}

{The first scenario involves 
a fixed-wing 2nd-order airplane moving through a building with
tight stairwells to reach the top floor (see Figure~\ref{fig:scenarios}, left). The state space is nine-dimensional.
The task space is the $x,y,z$ location of the
fixed-wing airplane. 
We present the results in Figure~\ref{fig:fixed_wing_results}.}
\begin{figure}[!t]
\centering
\input{fixedwing.tex}
\vspace{-5pt}
    \caption{Plots for 
    a fixed-wing airplane 
    (see Figure~\ref{fig:scenarios}, left). 
    }
\label{fig:fixed_wing_results}
\end{figure}
{An additional scenario (Figure~\ref{fig:scenarios}, right)
involves a rally car (green) moving through a parking lot trying to reach a parking space (yellow) while avoiding other static cars and obstacles. The state space is eight-dimensional, while the task space consists of the 2D pose ($x,y,\theta$) of the car.
We present the results in Figure~\ref{fig:rallycar_results}.}

\begin{figure}[ht]
\centering
\input{car.tex}
\vspace{-5pt}
    \caption{Plots for a rally car (see Figure~\ref{fig:scenarios}, right). 
    }
    \label{fig:rallycar_results}
\end{figure}

{These two experiments demonstrate that, as expected, all \aorrt variants improve their solution as a function of time. However, single-tree variants within the \aorrt framework perform better than the \aorrtrebuilding approach. 
This further justifies the dedicated analysis for the single-tree \aorrt.
Additionally, we observe that \aorrt and \aorrtprune, differing in the addition of a pruning step, find solutions of similar quality, while the former obtains a slightly better success rate. The variance of the solutions found by the hybrid planner is higher than that of the other approaches for lower success rates.
}

{Moreover, when compared to \rrt, which is not AO, all \aorrt variants were able to find solutions of better quality.
The comparison against \sst, which is near-AO, yielded different results; for the fixed-wing scenario all \aorrt variants had better success rates and obtained better costs. 
For the rally car, all the single-tree \aorrt variants and \sst found comparable solutions, with a slight advantage to \sst.}

}
{
We present an experimental evaluation of the performance
of \rrttwo on both geometric and kinodynamic scenarios.
Our experiments were conducted on an Intel(R) Xeon(R) CPU E5-1660 v3\@3.00GHz with 32GB of memory. 
We set the optimization objective to be the duration of the trajectory. 
Unless otherwise stated, throughout the experiments we assign $w_x = w_c$ in the distance metric 
(Eq.~\eqref{eq:dist_func}) for all \aorrt variants.

We first visualize the behaviour of \rrttwo (Alg.~\ref{alg:rrttwo}) in a simple geometric setting.
We run \rrttwo for 120 seconds in a simple 2D environment consisting of a disc robot moving among rectangular obstacles (see Figure~\ref{fig:rrtstar_geom_experiment}). The constants in Eq.~\eqref{eq:dist_func} were set to $w_x=1,w_c=0.2$.  We depict the paths found during the run. As the number of samples increases the paths improve, gradually converging to the optimal path.
\begin{figure}[t]
    \vspace{8pt}
    \centering
    \includegraphics[width=0.4\textwidth]{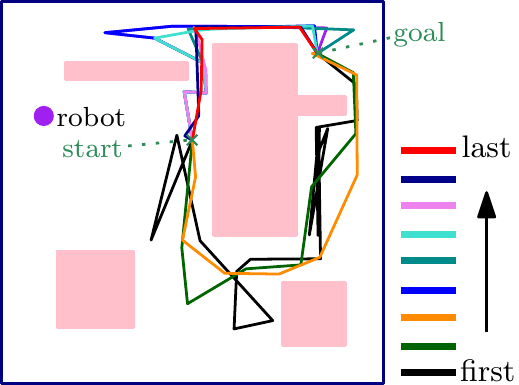}
    \caption{Planning for a disc robot in a 2D geometric setting using \rrttwo for 120 seconds. Obstacles are depicted in light magenta, while the disc robot is depicted in purple. Start and goal positions are marked with a green cross. The Paths found by \rrttwo during the fixed time budget are drawn. The paths converge to the optimum as the number of samples increases.}
    \vspace{-3pt}
    \label{fig:rrtstar_geom_experiment}
\end{figure}

Next, we compare the algorithms within the \aorrt framework with \rrt~\cite{LaVKuf01} and \sst~\cite{LiETAL16}, which have weaker guarantees. The \aorrt variants that we used are:
(i) \aorrtrebuilding---the algorithm analyzed in~\cite{HauserZ16}, (ii) \aorrtnopruning---an implementation of~Algorithm~\ref{alg:rrttwo}, (iii) \aorrtprune, which performs node pruning, and (iv) \hybrrttwostride. 
The latter is a hybrid planner, as described in Section~\ref{sec:practical_aspects}, combining \rrttwo and \stride~\cite{GipsonMK13}.
We mention that \stride was originally defined for geometric settings and was not shown to be AO. \stride uses a data structure that enables it
to produce density estimates in the full state
space. More precisely, it samples a configuration $s$, biased towards relatively unexplored areas of the state space. The tree is then grown from $s$, if possible. \hybrrttwostride maintains, as \stride does, a data structure for states in the augmented space $\Y$ and alternates between the two methods for choosing the node $\ynear$ to grow the tree from.
Once the node is chosen, the algorithm proceeds as \rrttwo does.

For each planner we report on both the success rate and the minimum solution cost averaged over all successful runs, displaying values within one standard deviation of the mean. Each result is averaged over 50~runs. Note that since we plot the average cost over all successful runs, we may observe for a certain planner an increase in the average cost. This is only possible if the success rate increases as well.

\begin{figure*}
    \centering
    \includegraphics[width=0.21\textwidth]{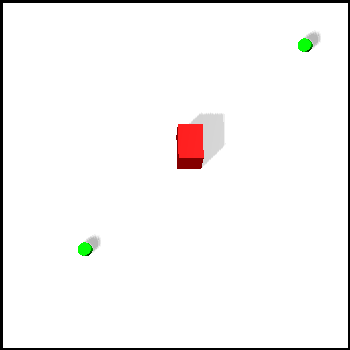}
    \includegraphics[width=0.21\textwidth]{fig/building_2.png} 
   \includegraphics[width=0.20\textwidth]{fig/rally_car.png}
    \vspace{-5pt}
    \caption{Scenarios: Geometric 2D point robot (left), fixed-wing airplane in a building (middle), and rally car in a parking lot (right).}
    \label{fig:kinoscenarios}
    \vspace{-5pt}
\end{figure*}

We begin with a simple geometric scenario involving a point robot translating in a 2D environment consisting of a single rectangular obstacle (see Figure~\ref{fig:kinoscenarios}, left).
Figure~\ref{fig:pointrobot_results} depicts the results.
Indeed, all \aorrt variants improve their solution as a function of time and are comparable in terms of performance. In fact, these variants were able to find the best solutions among all tested planners.
The hybrid planner obtains better solutions quicker, possibly due to the \stride component included in it that enhances the exploration in geometric settings.
\rrt was inferior in terms of cost and success rate when compared to the other planners. 

\begin{figure}[!t]
\centering
\input{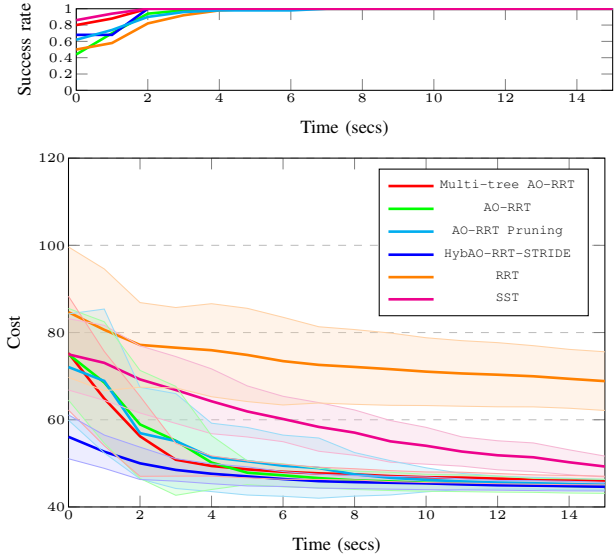}
\vspace{-5pt}
    \caption{Plots for 
     a geometric scenario: a point robot operating in a 2D environment (see Figure~\ref{fig:kinoscenarios}, left).
    }
\label{fig:pointrobot_results}
\end{figure}

Then we consider two kinodynamic scenarios. The first involves 
a fixed-wing 2nd-order airplane moving through a building with
tight stairwells to reach the top floor (see Figure~\ref{fig:kinoscenarios}, middle). The state space is nine-dimensional.
The task space is the $x,y,z$ location of the
fixed-wing airplane. 
We present the results in Figure~\ref{fig:fixed_wing_results}.
\begin{figure}[!t]
\centering
\input{fixedwing.tex}
\vspace{-5pt}
    \caption{Plots for 
    a fixed-wing airplane 
    (see Figure~\ref{fig:kinoscenarios}, middle). 
    }
\label{fig:fixed_wing_results}
\end{figure}

An additional scenario (Figure~\ref{fig:kinoscenarios}, right)
involves a rally car (green) moving through a parking lot trying to reach a parking space (yellow) while avoiding other static cars and obstacles. The state space is eight-dimensional, while the task space consists of the 2D pose ($x,y,\theta$) of the car.
We present the results in Figure~\ref{fig:rallycar_results}.

\begin{figure}[ht]
\centering
\input{car.tex}
\vspace{-5pt}
    \caption{Plots for a rally car (see Figure~\ref{fig:kinoscenarios}, right). 
    }
    \label{fig:rallycar_results}
\end{figure}

These two experiments demonstrate that, as expected, all \aorrt variants improve their solution as a function of time. However, single-tree variants within the \aorrt framework perform better than the \aorrtrebuilding approach. 
This further justifies the dedicated analysis for the single-tree \aorrt.
Additionally, we observe that \aorrt and \aorrtprune, differing in the addition of a pruning step, find solutions of similar quality, while the former obtains a slightly better success rate. The variance of the solutions found by the hybrid planner is higher than that of the other approaches for lower success rates.

Moreover, when compared to \rrt, which is not AO, all \aorrt variants were able to find solutions of better quality.
The comparison against \sst, which is near-AO, yielded different results; for the fixed-wing scenario all \aorrt variants had better success rates and obtained better costs. 
For the rally car, all the single-tree \aorrt variants and \sst found comparable solutions, with a slight advantage to \sst.

Finally, we present experiments examining the effect of the weighting scheme used in the state-cost space distance metric (Eq.~\eqref{eq:dist_func}).
We ran \aorrtprune with different weighting schemes on the fixed-wing airplane scenario (Figure~\ref{fig:kinoscenarios}, middle). We plot the success rate and the cost averaged over all successful runs in Figure~\ref{fig:fixedwing_weight_results}.
The plots demonstrate that the choice of weights may affect the convergence rate of the algorithm. Note that with weights $w_x=1.0, w_c = 0.0$  the algorithm acts like vanilla \rrt, showing almost no improvement in cost. However, all runs with $w_c>0$ were able to converge to the optimum.


\begin{figure}[ht]
\centering
\input{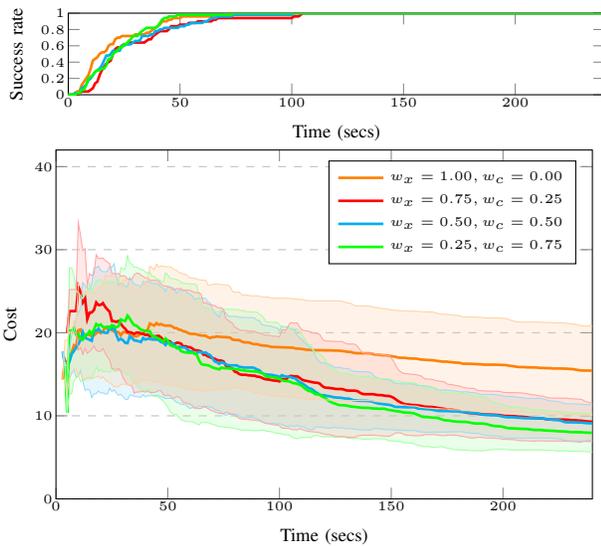}
\vspace{-5pt}
    \caption{The effect of the weights $w_x,w_c$ for a fixed-wing airplane (see Figure~\ref{fig:kinoscenarios}, middle).
    }
\label{fig:fixedwing_weight_results}
\end{figure}

}

\section{Discussion}
\label{sec:discussion}
This work analyzed the desirable theoretical properties of \rrttwo, which is a method for kinodynamic sampling-based motion planning. In particular, relaxed sufficient conditions have been identified for which \rrttwo is asymptotically optimal. 

Future work will extend the framework to manifold-type constraints. In this case, one has to consider the notion of Riemannian distance instead of the more classic Euclidean distance that we use here.

We noticed that \rrttwo's performance strongly depends on the choice of weights $w_x,w_c$ used by the distance function in Eq.~\eqref{eq:dist_func}. It would be desirable to come up with an automatic scheme to choose these weights, or even modify them on-the-fly in order to obtain favorable results.

Another possible research direction involves a deeper examination of the properties of the hybrid approach. This could shed light on the settings in which the hybrid planner has an advantage over \rrttwo.

Finally, the following question, concerning the sampling scheme used by the algorithm, arises from this work:  Is it possible to replace the uniform sampling of durations, or states with a different sampling method to converge more quickly to the optimum in the state-cost space?

\vspace{-5pt}

%
%
\bibliographystyle{IEEEtran}
\bibliography{bibliography}

\end{document}